\documentclass[11pt]{article}
\usepackage{inputenc}
\usepackage[english]{babel}
\usepackage{color}
\usepackage{algorithm}
\usepackage{tikz-cd}
\usepackage{algorithmic}
\usepackage{hyperref}
\usepackage[letterpaper,top=2cm,bottom=2cm,left=3cm,right=3cm,marginparwidth=1.75cm]{geometry}
\usepackage{mathrsfs}
\usepackage{amsmath}
\usepackage{amsfonts}
\usepackage{amsthm}
\usepackage{tikz}

\newcommand{\calP}{\mathcal{P}}

\newcommand{\calE}{\mathcal{E}}
\newcommand{\calV}{\mathcal{V}}

\newcommand{\calX}{\mathcal{X}}

\newcommand{\E}{\mathbb{E}}

\newcommand{\R}{\mathbb{R}}
\newcommand{\bH}{\mathbb{H}}
\newcommand{\bP}{\mathbb{P}}
\newcommand{\tG}{\tilde{G}}

\newtheorem{prop}{Proposition}

\newtheorem{defi}{Definition}
\newtheorem{thm}{Theorem}

\newtheorem{rmk}{Remark}
\newtheorem{assumption}{Assumption}

\title{Stochastic Gradient Descent over $\mathcal{P}_2$}

\author{
Maria Oprea\thanks{Center for Applied Mathematics, Cornell University, \url{moprea@ista.ac.at}, corresponding author} \thanks{Current address: Institute of Science and Technology Austria (ISTA)} 
\and Qin Li\thanks{Department of Mathematics, University of Wisconsin--Madison, USA, \url{qinli@math.wisc.edu}}
\and Yunan Yang\thanks{Department of Mathematics, Cornell University, USA, \url{yunan.yang@cornell.edu}}
}

\date{}

\begin{document}

\maketitle

\begin{abstract}
Stochastic gradient descent (SGD) admits diffusion approximations that replace the complicated randomness of stochastic gradients by Gaussian noise, providing a powerful tool for understanding its dynamics and long-time behavior. We investigate whether an analogous approximation principle holds for optimization over probability measures, where the objective is a functional defined on the Wasserstein space $\mathcal P_2$. The nonlinear geometry and infinite-dimensional nature of $\mathcal P_2$ prevent a direct extension of the classical Euclidean theory. Using Lions differentiability, we lift the problem to a linear Hilbert space, where higher-order differential calculus becomes available. We then construct a Gaussian random-field approximation whose velocity field matches the mean and covariance of the original stochastic gradient. By exploiting this moment matching through higher-order Taylor expansions, we show that the Gaussian approximation captures the SGD dynamics with second-order weak accuracy. Our result provides a rigorous foundation for replacing sample-driven randomness by analytically tractable Gaussian fluctuations in stochastic optimization over probability measures. 
\end{abstract}
\section{Introduction}\label{sec:introduction}
Stochastic gradient descent (SGD) is one of the fundamental algorithms in modern machine learning and large-scale optimization. Given an expected loss
\begin{equation}
\label{eq:finite_obj}
\min_{x \in \mathbb{R}^d}
F(x):=\mathbb{E}_{\gamma}[f(x,\gamma)],
\end{equation}
SGD generates iterates according to
\begin{equation}
\label{eq:SGD_classic}
x_{n+1} =x_n -h_n \nabla_x f(x_n,\gamma_n),
\end{equation}
where ${\gamma_n}$ are independent samples and $h_n$ is the step size. Since $\mathbb{E}_{\gamma}\big[\nabla_x f(x,\gamma)\big]=\nabla_x F(x)$, the stochastic gradient provides an unbiased approximation of the exact gradient, and \eqref{eq:SGD_classic} may be viewed as a stochastic counterpart of the classical gradient descent iteration
\begin{equation}\label{eq:GD_classic}
x_{n+1}
=
x_n
-
h_n
\nabla_x F(x_n).
\end{equation}
Introduced by Robbins and Monro~\cite{robbins1951stochastic}, SGD enables efficient optimization in high-dimensional settings and forms the computational backbone of contemporary machine learning~\cite{bottou2018optimization}.

Beyond computational efficiency, the introduced stochasticity plays a crucial role in the success of SGD. The noise introduced through random sampling often improves exploration of the optimization landscape and can lead to superior practical performance compared with deterministic gradient methods. Understanding the effect of this stochasticity has motivated extensive theoretical investigations. One particularly influential viewpoint studies continuous-time limits of SGD. Under suitable scalings, SGD can be approximated by a stochastic differential equation (SDE)
\begin{equation}\label{eq:SDE_classic}
\mathrm{d}X_t = - \nabla F(X_t)\mathrm{d}t +\sqrt{2D(t,X_t)},\mathrm{d}W_t\,,
\end{equation}
where the diffusivity coefficient $D$ depends on the stochastic gradient noise. This perspective connects optimization with Langevin dynamics, simulated annealing, and Fokker--Planck equations~\cite{geman1986diffusions,chiang1987diffusion,kushner1987asymptotic,hwang1990large,gelfand1991recursive}, providing analytical tools that have led to a deeper understanding of the behavior of SGD~\cite{cheng2020stochastic}.

The connection between \eqref{eq:SGD_classic} and \eqref{eq:SDE_classic} is far from obvious. The randomness in SGD originates from sampling the random variable $\gamma$, whereas the randomness in the SDE is generated by Brownian motion and is therefore Gaussian. These stochastic mechanisms are fundamentally different, and the diffusion process can only be viewed as an approximation of the discrete algorithm. Establishing the validity of such diffusion approximations, and quantifying the associated approximation error, have been the subject of considerable study in stochastic approximation theory and stochastic modified equations; see, for example,~\cite{li2019stochastic}.

While classical SGD is formulated in finite-dimensional Euclidean spaces, many contemporary problems require optimization over probability distributions. Examples arise in variational inference, Bayesian sampling~\cite{jordan1998variational,SVGD_Liu}, generative modeling~\cite{XieCheng_generative,han2026onestepgenerativemodelingwasserstein}, and mean-field systems~\cite{DingChenLiWright,Rotskoff_VE,chizat2018global,geshkovski2025mathematical}. In such settings, the optimization variable is no longer a finite-dimensional parameter vector but a probability measure itself. This naturally leads to problems of the form
\begin{equation}\label{eq:problem}
\min_{\rho\in\mathcal P_2(\mathcal X)}
E(\rho)\,,\qquad\text{with}\qquad E(\rho) :=\mathbb E_\gamma\big[\mathcal E_\gamma(\rho)\big]\,,
\end{equation}
where $\mathcal P_2(\mathcal X)$ denotes the space of probability measures on $\mathcal X$ with finite second moment, $\gamma$ is a random parameter defined on a probability space $(\Gamma,\mathcal F,\mathbb P)$, and $\mathcal E_\gamma:\mathcal P_2(\mathcal X)\to\mathbb R$ is a sample-dependent loss functional. This definition is an analogue of~\eqref{eq:finite_obj}.

Extending the SGD-SDE paradigm from Euclidean spaces to probability spaces introduces several additional difficulties. The first is geometric. Unlike $\mathbb R^d$, the space $\mathcal P_2(\mathcal X)$ is not linear, and therefore the additive update underlying classical SGD is no longer meaningful. Fortunately, this obstacle has largely been resolved through the development of optimal transport and Wasserstein gradient-flow theory. Following the framework of Ambrosio, Gigli, and Savar\'e~\cite{ambrosio2005gradient}, one may endow $\mathcal P_2(\mathcal X)$ with the Wasserstein metric and define gradients and gradient flows for sufficiently regular functionals.

The second difficulty concerns the nature of the stochasticity. In SGD, randomness enters through independent draws of the sample variable $\gamma$ from its underlying distribution, and the resulting stochastic gradient noise is generally non-Gaussian. By contrast, the continuous-time diffusion model is driven by Brownian motion, or, in the probability-space setting, by a Gaussian random field. Thus the two dynamics are expected to differ pathwise. Following the stochastic modified equation perspective of~\cite{li2019stochastic}, we instead compare the discrete stochastic iteration and its diffusion approximation in a weak sense. This replaces trajectory-wise accuracy by accuracy of expectations against suitable test functionals, for which matching the drift and covariance of the stochastic increments becomes the central requirement.

A third difficulty arises when one attempts to extend the stochastic modified equation framework itself to probability spaces. In the finite-dimensional setting, the work of~\cite{li2019stochastic} establishes a continuous-in-time SDE whose weak solutions provide a high-order approximation of the SGD iterates. Repeating the same program in $\mathcal P_2(\mathcal X)$ is considerably more delicate. Since probability measures are inherently infinite-dimensional objects, the corresponding continuous-time limit is no longer an SDE but rather a stochastic partial differential equation (SPDE). Within the Wasserstein gradient-flow framework, a formal derivation suggests an SPDE driven by a Gaussian velocity field whose covariance is determined by the stochastic gradient fluctuations. At present, however, the well-posedness and regularity theory for such dynamics remain largely unresolved. Consequently, rather than constructing the limiting SPDE directly, we develop the theory at the discrete-time level and use weak approximation arguments to characterize the stochastic dynamics.

Combining these observations, we arrive at a natural analogue of SGD on the Wasserstein space $\mathcal P_2(\mathcal X)$. Following the geometric framework of Wasserstein gradient flows, each stochastic realization generates a velocity field through the first variation of the sample-dependent objective. The resulting stochastic update takes the form:
\begin{equation}
\label{eq:JKO_SGD}
\rho_{n+1} = \Bigl(\mathrm{Id} - h u_{\gamma_n}(\cdot,\rho_n) \Bigr)_\# \rho_n,\quad\text{with}\quad u_\gamma(x,\rho):= \nabla_x \frac{\delta\mathcal E_\gamma}{\delta\rho}(\rho)(x)\,,
\end{equation}
where $(\gamma_n)_{n\ge0}$ are independent samples drawn from $(\Gamma,\mathcal F,\mathbb P)$ and $\#$ denotes the pushforward operator \cite{ambrosio2005gradient}. This iteration is the natural measure-valued analogue of classical SGD. Its relation to deterministic Wasserstein gradient descent mirrors the relation between \eqref{eq:SGD_classic} and \eqref{eq:GD_classic} in Euclidean optimization.

Likewise, following the stochastic modified equation philosophy of~\cite{li2019stochastic}, we seek a Gaussian approximation that matches the first two moments of the stochastic velocity field. This leads to the auxiliary iteration:
\begin{equation}\label{eq:JKO_Gaussian}
\widetilde{\rho}_{n+1}= \left(\mathrm{Id}-hu_{\omega_n}(\cdot,\widetilde{\rho}_n)\right)_\# \widetilde{\rho}_n\,,
\end{equation}
where, conditional on $\widetilde{\rho}_n$, the random field $u_{\omega_n}(\cdot,\widetilde{\rho}_n)$ is a Gaussian field that share the mean and the covariance with $u_{\gamma}$.

The central objective of this work is to establish a rigorous connection between the stochastic iteration \eqref{eq:JKO_SGD} and its Gaussian approximation \eqref{eq:JKO_Gaussian}. Since the two velocity fields share the same first and second moments, one naturally expects the two dynamics to agree up to second order in the time step. Our main result confirms this intuition. Adopting the weak-approximation framework of~\cite{li2019stochastic}, we compare the two processes through sufficiently regular observables and prove a second-order weak error estimate. More precisely, for a suitable class of test functionals on $\mathcal P_2(\mathcal X)$, the discrepancy between the two updates is shown to be of order $\mathcal O(h^2)$. As in the classical theory of stochastic modified equations, the associated error bound grows exponentially with the final time.

The proof relies on two key ingredients. The first is the Wasserstein differential calculus developed through Lions derivatives. By lifting functionals on $\mathcal P_2(\mathcal X)$ to an underlying Hilbert space, Lions differentiability provides a convenient framework for defining higher-order derivatives and carrying out the Taylor expansions required for moment matching. The second ingredient is an iterative weak-error representation formula. In the finite-dimensional setting, a related construction was introduced in~\cite{li2019stochastic} and plays a crucial role in controlling the accumulation of local approximation errors. A similar mechanism is employed here. As in the Euclidean theory, repeated iterations lead to a progressive loss of regularity of the test functional, which must be compensated through the smoothing properties of the updates themselves.

The remainder of the paper is organized as follows. Section~\ref{sec:background} reviews the Wasserstein geometry and Lions differential calculus that form the analytical foundation of our approach. Section~\ref{sec:main} introduces the Gaussian approximation and establishes the main weak-convergence results. Section~\ref{sec:numerics} presents numerical experiments illustrating the second-order weak error in both a solvable benchmark and a state-dependent distribution-reconstruction problem. Section~\ref{sec:conclusion} concludes with limitations and directions for future work.

\section{Mathematical background}\label{sec:background}

This section collects the background material used throughout the weak-convergence analysis. The proof combines four ingredients. Wasserstein geometry provides the gradient-flow structure underlying the stochastic update. Lions differentiability supplies a differential calculus suitable for higher-order Taylor expansions on probability spaces. A covariance identity for bounded bilinear forms allows us to match the second-order fluctuations of the stochastic and Gaussian velocity fields. Finally, we introduce a class of sufficiently regular test functionals on $\calP_2(\calX)$ against which the weak error will be measured.

Throughout the paper, we assume that $\calX\subset\R^d$ is compact and convex. We write
\[
R_\calX:=\sup_{x\in\calX}|x|.
\]
Since $\calX$ is compact, every probability measure on $\calX$ has finite second moment. Consequently, $\calP_2(\calX)=\calP(\calX)$ are two equivalent sets. We nevertheless retain the notation $\calP_2(\calX)$ to emphasize the Wasserstein geometry.

\subsection{Wasserstein gradient flows}\label{subsec:W2_geo}

The geometric structure of $\calP_2(\calX)$ is induced by the Wasserstein metric. A particularly useful characterization is the Benamou--Brenier dynamic formulation; see \cite[Chapter 8]{topicsOT}:
\[
W_2^2(\mu,\nu)
= \inf_{(\mu_t,v_t)}
\left\{
\int_0^1\int_{\calX}|v_t(x)|^2,d\mu_t(x)dt\middle|
\partial_t\mu_t+\nabla\cdot(\mu_t v_t)=0,\
\mu_0=\mu,\
\mu_1=\nu
\right\}.
\]
Here, the continuity equation is understood in the weak sense, together with the usual no-flux boundary condition whenever $\calX$ has boundary. This formulation highlights the velocity-field representation of transport, which forms the basis of the gradient-flow construction.

A key consequence of the Wasserstein geometry is that sufficiently regular functionals admit a notion of gradient. Let
\[
\calE:\calP_2(\calX)\to\R
\]
be a smooth functional. Then its Wasserstein gradient can be expressed through the first variation as
\[
\nabla^{W_2}\calE(\rho) =-\nabla\cdot\left(\rho\nabla \left. \frac{\delta\calE}{\delta\rho} \right|_\rho\right)\,,
\]
where $\frac{\delta\calE}{\delta\rho}$ denotes the first variation of $\calE$; see, for example, \cite[Proposition 2.2]{natural_gradient}. Throughout the paper, we identify the Wasserstein gradient flow as:
\[
\partial_t\rho_t=-\nabla^{W_2}\calE(\rho_t)= \nabla\cdot \Bigl( \rho_t u(\cdot,\rho_t) \Bigr)\,,\quad\text{with}\quad u(x,\rho) := \nabla_x
\left.
\frac{\delta\calE}{\delta\rho}
\right|_\rho(x)\,.
\]
Applying a forward-Euler discretization to this evolution yields the deterministic update
\[
\rho_{n+1}=\Bigl( \mathrm{Id}-h u(\cdot,\rho_n) \Bigr)_\# \rho_n\,.
\]

Returning to the stochastic optimization problem \eqref{eq:problem}, each realization $\gamma$ gives rise to a sample-dependent functional $\calE_\gamma$ and therefore $\nabla^{W_2}\calE_\gamma(\rho)$ gives rise to a corresponding velocity field
\begin{equation}
\label{eq:u_gamma_background}
u_\gamma(\cdot,\rho):=\nabla_x \left. \frac{\delta\calE_\gamma}{\delta\rho}\right|_\rho \,.
\end{equation}
Replacing the deterministic velocity field in the forward-Euler scheme by a single sample realization $u_{\gamma_n}$ yields precisely the stochastic iteration~\eqref{eq:JKO_SGD}.



\subsection{Lions differentiability and Taylor expansion}\label{sec:lions}

The analysis developed later relies heavily on higher-order Taylor expansions of functionals defined on $\calP_2(\calX)$. While Wasserstein geometry provides a natural notion of gradient, repeatedly differentiating functionals directly on the space of probability measures is considerably less convenient. A powerful alternative is provided by Lions differentiability, which lifts functionals on $\calP_2(\calX)$ to a linear Hilbert space where classical Fréchet calculus becomes available; see~\cite{gangbo1}, \cite[Chapter 5.2]{lions}, and~\cite{notes_on_mfg}.

Let $(\Theta,\mathcal F_\theta,\bP_\theta)$ be a nonatomic probability space, and define
\[
\bH:=L^2(\Theta,\bP_\theta;\R^d)\,,\quad\text{equipped with}\quad 
\langle X,Y\rangle_{\bH}
:=
\int_\Theta
\langle X(\theta),Y(\theta)\rangle \,d\bP_\theta(\theta).
\]
Since $\bP_\theta$ is nonatomic,  and $\calX \subset \R^d$, every measure $\mu\in\calP_2(\calX)$ can be realized as the law of some random variable $X\in\bH$; see~\cite[Theorem 3.18]{santambrogio},  with $X(\Theta) \subseteq \calX$. That is, $\mu =X_\#\bP_\theta$. Consequently, a functional
\[
G:\calP_2(\calX)\to\R
\]
can be lifted to a functional on $\bH$ by defining
\[
\tG(X) := G(X_\#\bP_\theta), \qquad X\in\bH\,.
\]
The key advantage of this construction is that the domain of $\tG$ is now a Hilbert space, allowing one to deploy the familiar machinery of Fr\'echet differentiation.

A subtlety arises from the fact that the lift is not unique. Indeed, multiple random variables may induce the same probability measure: $X_1{}_\#\bP_\theta= X_2{}_\#\bP_\theta$. The lifted functional itself is law invariant, namely
\[
\tG(X_1) =\tG(X_2)\,, \qquad \text{whenever } X_1{}_\#\bP_\theta=X_2{}_\#\bP_\theta\,.
\]
However, it is not immediately clear that the corresponding Fr\'echet derivatives inherit the same invariance property. In principle, one could have
\[
\left. \frac{\delta\tG}{\delta X} \right|_{X_1} \neq \left. \frac{\delta\tG}{\delta X} \right|_{X_2}\,, \qquad X_1{}_\#\bP_\theta=X_2{}_\#\bP_\theta\,.
\]
If this occurs, the differential calculus on $\bH$ would depend on the particular choice of lift and therefore fail to define an intrinsic notion of differentiability on $\calP_2(\calX)$. Establishing law invariance of the derivatives is therefore essential for connecting the lifted calculus with the geometry of probability measures.

This issue is resolved by the classical theory of Lions differentiability. In particular, \cite[Theorem 5.24]{lions} and \cite[Theorem 6.2]{notes_on_mfg} imply that, whenever a law-invariant lift $\tilde G$ is Fr\'echet differentiable, its derivative is itself law invariant and therefore depends only on the measure $\mu=X_\#\bP_\theta$. Consequently, Fr\'echet differentiability of the lift induces a well-defined notion of differentiability on $\calP_2(\calX)$.

\begin{defi}[Lions differentiability]\label{def:lions}
Let $\mu\in\calP_2(\calX)$ and let $X\in\bH$ be its lift satisfying $X_\#\bP_\theta=\mu$. We say that $G:\calP_2(\calX)\to\R$ is Lions differentiable at $\mu$ if its lift $\tG$ is Fr\'echet differentiable at $X$. By the preceding law-invariance result, there exists $\partial_\mu G(\mu):\calX\to\R^d$ such that
\[
\partial_\mu G(\mu)(X)=d\tG_X\,,\quad\text{for all}\quad X_\#\mathbb{P}_\theta=\mu\,,
\]
where $d\tG_X$ denotes the Fr\'echet derivative of $\tG$ evaluated at $X$. 
This function is called the Lions derivative of $G$ at $\mu$. When the first variation exists and is sufficiently regular, it coincides with the velocity-field generating the Wasserstein gradient:
\[
\partial_\mu G(\mu)(x) = \nabla_x
\left.\frac{\delta G}{\delta\rho}\right|_{\mu}(x)\,,\quad\text{so that}\quad\nabla^{W_2}G(\mu)=-\nabla\cdot\left(\mu\partial_\mu G(\mu)\right) \,.
\]
\end{defi}

For the purpose of defining gradients on probability spaces, first-order differentiability is sufficient. The analysis developed later, however, relies on repeated Taylor expansions and therefore requires higher-order derivatives. We will need an analogue of the above law-invariance property beyond first order. The following proposition establishes precisely this extension.

\begin{prop}\label{prop:independence_on_lift}
Let $G:\calP_2(\R^d)\to\R$, and suppose that its lift $\tG:\bH\to\R$ is $k$ times Fr\'echet differentiable. Let $\mu\in\calP_2(\R^d)$ and $v_1,\ldots,v_k\in L^2(\mu;\R^d)$. Then
\[
\partial_\mu^kG[v_1,\ldots,v_k]
:=
d_X^k\tG[v_1\circ X,\ldots,v_k\circ X]
\]
is independent of the choice of lift $X\in\bH$ satisfying $X_\#\bP_\theta=\mu$. Equivalently, whenever
\[
{X_1}_\#\bP_\theta={X_2}_\#\bP_\theta=\mu,
\]
one has
\[
d_X^k\tG[v_1\circ X_1,\ldots,v_k\circ X_1]
=
d_X^k\tG[v_1\circ X_2,\ldots,v_k\circ X_2].
\]
\end{prop}
The proof is elementary, but since we were unable to locate a corresponding statement in the literature, we include it for completeness.
\begin{proof}[Proof of Proposition~\ref{prop:independence_on_lift}]
For $t=(t_1,\ldots,t_k)\in\R^k$, define
\[
T_t(x)=x+\sum_{i=1}^k t_i v_i(x).
\]
Since $X_1$ and $X_2$ have the same law $\mu$,
\[
\left(X_1+\sum_{i=1}^k t_i v_i\circ X_1\right)_\#\bP_\theta=(T_t)_\#\mu = \left(X_2+\sum_{i=1}^k t_i v_i\circ X_2\right)_\#\bP_\theta.
\]
The law invariance of $\tG$ therefore gives
\[
\tG\left(X_1+\sum_{i=1}^k t_i v_i\circ X_1\right)
=
\tG\left(X_2+\sum_{i=1}^k t_i v_i\circ X_2\right)
\]
for every $t\in\R^k$. Taking the mixed derivative $\partial_{t_1}\cdots\partial_{t_k}$ at $t=0$ yields
\[
d_X^k\tG[v_1\circ X_1,\ldots,v_k\circ X_1]
=
d_X^k\tG[v_1\circ X_2,\ldots,v_k\circ X_2].
\]
Hence $\partial_\mu^kG$ is well defined independently of the chosen lift of $\mu$.
\end{proof}





This measure-invariance property for all high-order derivatives allows one to perform Taylor expansion, and we summarize it below.

\begin{thm}[Taylor expansion in the Lions sense]\label{thm:Taylor}
Let $G:\calP_2(\calX)\to\R$ admit a law-invariant lift $\tG:\bH\to\R$. Let $X, H\in\bH$. Assume that $\tG$ is $C^3$ on an open neighborhood of the segment
\[
\{X+\tau H:\tau\in[0,1]\}.
\]
Define
\[
\mu:=X_\#\bP_\theta,
\qquad
\widetilde\mu:=(X+H)_\#\bP_\theta.
\]
Then
\begin{equation}\label{eq:taylor_lions}
G(\widetilde\mu)
=
G(\mu)
+
d_X\tG[H]
+
\frac12 d_X^2\tG[H,H]
+
R_3(X,H),
\end{equation}
where
\begin{equation}\label{eq:remainder_lions}
R_3(X,H)
=
\int_0^1
\frac{(1-\tau)^2}{2}
d_{X+\tau H}^3\tG[H,H,H]\,d\tau .
\end{equation}
Consequently, if $\sup_{\tau\in[0,1]} \|d_{X+\tau H}^3\tG\|_{\mathrm{op}} \le M$, then
\begin{equation}\label{eq:taylor_remainder_bound}
|R_3(X,H)|
\le
\frac{M}{6}\|H\|_{\bH}^3.
\end{equation}
\end{thm}

\begin{proof}
Set
\[
\phi(\tau):=\tG(X+\tau H).
\]
By assumption, $\phi$ is $C^3$ on $[0,1]$, with
\begin{eqnarray*}
\phi'(\tau)=&d_{X+\tau H}\tG[H],\\
\phi''(\tau)=&d_{X+\tau H}^2\tG[H,H],\\
\phi^{(3)}(\tau)=&d_{X+\tau H}^3\tG[H,H,H].
\end{eqnarray*}
The one-dimensional Taylor formula with integral remainder gives
\[
\phi(1)
=
\phi(0)+\phi'(0)+\frac12\phi''(0)
+
\int_0^1\frac{(1-\tau)^2}{2}\phi^{(3)}(\tau)\,d\tau.
\]
Using
\[
\phi(0)=G(X_\#\bP_\theta),
\qquad
\phi(1)=G((X+H)_\#\bP_\theta),
\]
gives \eqref{eq:taylor_lions} and \eqref{eq:remainder_lions}. The bound \eqref{eq:taylor_remainder_bound} follows from the definition of the operator norm and the identity $\int_0^1\frac{(1-\tau)^2}{2}\,d\tau=\frac16$.
\end{proof}

\subsection{Moment matching and Gaussian realization}\label{subsec:moment_matching}

A central objective of the present work is to compare the stochastic iteration~\eqref{eq:JKO_SGD} with its Gaussian counterpart~\eqref{eq:JKO_Gaussian}. The two updates are driven by fundamentally different sources of randomness: the former samples from the distribution of $\gamma$, whereas the latter samples from a Gaussian field. Exact agreement between the two random processes is therefore impossible in general. Instead, following the philosophy of stochastic modified equations, we seek an approximation that preserves the first few moments of the stochastic velocity field.

More precisely, for a fixed measure $\rho\in\calP_2(\calX)$, define the mean velocity field and the covariance kernel:
\[
\begin{cases}
\bar u_\rho(x) := \mathbb E_\gamma[u_\gamma(x,\rho)]\\
K_\rho(x,y) := \mathbb E_\gamma \Big[ \bigl(u_\gamma(x,\rho)-\bar u_\rho(x)\bigr) \otimes \bigl(u_\gamma(y,\rho)-\bar u_\rho(y)\bigr) \Big]\,.
\end{cases}
\]
The Gaussian field $u_\omega(\cdot,\rho)$ is then defined to be the Gaussian random field with mean $\bar u_\rho$ and covariance kernel $K_\rho$.

One convenient representation is provided by the Karhunen--Lo\'eve expansion:
\[
u_\omega(\cdot,\rho) :=U(x,\rho;\xi)=\bar u_\rho + \sum_{i=1}^{\infty}
\sqrt{\lambda_i(\rho)}\,\phi_i^\rho\,\xi_i\,,
\]
where $(\xi_i)_{i\ge1}$ are independent standard normal random variables and ${(\lambda_i(\rho),\phi_i^\rho)}_{i\ge1}$ denotes the eigensystem of the covariance operator associated with $K_\rho$. The notation $U(x,\rho;\xi)$ is introduced to emphasize the realization of the randomness through the countably many random variables $\xi$. Throughout this paper, $\xi$ will be exclusively used to denote randomness that is standard normal, while $\omega$ is used to denote the general Gaussian randomness. Moreover, $\Pi_\rho$ denotes the Gaussian measure with mean $\bar{u}_\rho$ and covariance kernel $K_\rho$, while $\E_\omega$ denotes the expectation with respect to $\Pi_\rho$.

The significance of this construction lies in the fact that the Gaussian field and the original stochastic velocity field possess identical first and second moments. Since the arguments developed later rely on second-order Taylor expansions, only these moments enter the leading-order error analysis. Consequently, quadratic expressions generated by the expansion should agree whenever the underlying random fields share the same mean and covariance.

The following proposition formalizes this observation in an abstract Hilbert-space setting. Notably, Gaussianity itself plays no role in the statement; only moment matching is required.

\begin{prop}\label{prop:bilinear_forms_and_covariance}
Let $\bH$ be a real Hilbert space, and let
\[
B:\bH\times\bH\to\R
\]
be a bounded bilinear form. Let $U$ and $V$ be square-integrable $\bH$-valued random variables, possibly defined on different probability spaces. Assume that $U$ and $V$ have the same mean $m\in\bH$ and the same covariance in the sense that, for all $h_1,h_2\in\bH$,
\[
\mathbb E\Big[ \langle U-m,h_1\rangle_{\bH} \langle U-m,h_2\rangle_{\bH} \Big]=\mathbb E \Big[ \langle V-m,h_1\rangle_{\bH} \langle V-m,h_2\rangle_{\bH}\Big]\,.
\]
Then
\[
\mathbb E[B(U,U)]=\mathbb E[B(V,V)]\,.
\]
\end{prop}

\begin{proof}
Since $B$ is a bounded bilinear form on $\bH$, there exists a unique bounded operator $A\in\mathcal L(\bH)$ such that
\[
B(h_1,h_2)=\langle Ah_1,h_2\rangle_{\bH}.
\]
Let
\[
\widetilde U:=U-m,
\qquad
\widetilde V:=V-m.
\]
Then
\[
\mathbb E[B(U,U)]
=
B(m,m)+\mathbb E[B(\widetilde U,\widetilde U)]
+
\mathbb E[B(m,\widetilde U)]
+
\mathbb E[B(\widetilde U,m)].
\]
Since $\mathbb E[\widetilde U]=0$, the two mixed terms vanish. Thus
\[
\mathbb E[B(U,U)]
=
B(m,m)+\mathbb E\langle A\widetilde U,\widetilde U\rangle_{\bH}.
\]
Similarly,
\[
\mathbb E[B(V,V)]
=
B(m,m)+\mathbb E\langle A\widetilde V,\widetilde V\rangle_{\bH}.
\]

Let $(e_j)_{j\ge1}$ be an orthonormal basis of $\bH$. Since $U$ and $V$ are square-integrable, their covariance operators are trace class. The equality of covariances implies that
\[
\mathbb E
\big[
\langle \widetilde U,e_i\rangle_{\bH}
\langle \widetilde U,e_j\rangle_{\bH}
\big]
=
\mathbb E
\big[
\langle \widetilde V,e_i\rangle_{\bH}
\langle \widetilde V,e_j\rangle_{\bH}
\big]
\]
for all $i,j$. Therefore the covariance operators of $\widetilde U$ and $\widetilde V$ are equal. Calling this common covariance operator $K$, we have
\[
\mathbb E\langle A\widetilde U,\widetilde U\rangle_{\bH}
=
\operatorname{tr}(AK)
=
\mathbb E\langle A\widetilde V,\widetilde V\rangle_{\bH}.
\]
Hence, $\mathbb E[B(U,U)] =\mathbb E[B(V,V)]$.

\end{proof}

\section{Convergence result}\label{sec:main}
With the technical preparation from Section~\ref{sec:background}, we are now ready to state and prove the main convergence result. Throughout this section, we rewrite the stochastic Wasserstein gradient descent~\eqref{eq:JKO_SGD} and its Gaussian approximation~\eqref{eq:JKO_Gaussian} as
\begin{equation}\label{eq:discr}
\rho_j\color{black} := \rho_j^Z=\left(
\mathrm{Id}-hu_{\gamma_j}(\cdot,\rho_{j-1}^Z)
\right)_\# \rho_{j-1}^Z, \qquad
j=1,\ldots,n,
\end{equation}
and
\begin{equation}\label{eq:gauss}
 \tilde{\rho}_j := \color{black}\rho_j^X=\left(
\mathrm{Id}-hu_{\omega_j}(\cdot,\rho_{j-1}^X)
\right)_\# \rho_{j-1}^X,\qquad j=1,\ldots,n.
\end{equation}
The superscripts $Z$ and $X$ are used to distinguish the two processes. Both evolutions start from the same initial condition
\begin{equation}\label{eq:initialization}
\rho_0^Z=\rho_0=\rho_0^X.
\end{equation}

The proof is guided by a stability-consistency decomposition. A direct comparison between the two iterations yields
\begin{equation}
\begin{aligned}
\rho_n^Z-\rho_n^X
=&\left(
\mathrm{Id} -hu_{\gamma_n}(\cdot,\rho_{n-1}^Z)
\right)_\# \rho_{n-1}^Z-\left(\mathrm{Id}-hu_{\omega_n}(\cdot,\rho_{n-1}^X) \right)_\# \rho_{n-1}^X
\\
=& \underbrace{\left(\mathrm{Id}-hu_{\gamma_n} (\cdot,\rho_{n-1}^Z)\right)_\# \rho_{n-1}^Z-\left(\mathrm{Id}-hu_{\gamma_n}(\cdot,\rho_{n-1}^Z)
\right)_\# \rho_{n-1}^X}_{\mathrm{Term\ I}}
\\
&+\underbrace{ \left(\mathrm{Id}-hu_{\gamma_n}(\cdot,\rho_{n-1}^Z)\right)_\# \rho_{n-1}^X-\left(
\mathrm{Id}-hu_{\gamma_n}(\cdot,\rho_{n-1}^X) \right)_\#
\rho_{n-1}^X }_{\mathrm{Term\ II}}
\\
&+ \underbrace{\left(\mathrm{Id}-hu_{\gamma_n}(\cdot,\rho_{n-1}^X) \right)_\#\rho_{n-1}^X-\left(\mathrm{Id}-hu_{\omega_n}(\cdot,\rho_{n-1}^X)\right)_\# \rho_{n-1}^X}_{\mathrm{Term\ III}}.
\end{aligned}\label{eq:decomposition}
\end{equation}

The decomposition in~\eqref{eq:decomposition} already reveals the main structure of the proof. Term~I propagates discrepancies between two measures through the same velocity field, while Term~II propagates the same measure through velocity fields evaluated at different states. Under suitable regularity assumptions, both terms behave as stability terms and lead only to a mild amplification of previously accumulated errors. Namely $\rho_n^Z-\rho_n^X \sim (1+Ch)(\rho_{n-1}^Z-\rho_{n-1}^X)$ for some $C$. The true approximation error is generated by Term~III, where the stochastic velocity field $u_\gamma$ is replaced by its Gaussian realization $u_\omega$. Consequently, the convergence analysis reduces to establishing a sufficiently accurate estimate for Term~III and controlling the accumulation of these local errors over multiple iterations.

Two additional ideas are required to turn this heuristic decomposition into a rigorous argument. First, the comparison cannot be carried out directly at the level of measures, as is done in~\eqref{eq:decomposition}. Instead, we compare the readings of sufficiently smooth test functionals acting on the two processes. This allows us to exploit the moment-matching property of $u_\gamma$ and $u_\omega$ through higher-order Taylor expansions. Second, the consistency error represented by Term~III is regenerated at every iteration. To accumulate these local estimates efficiently, it is convenient to restart the dynamics from arbitrary intermediate states and propagate observables backward through the iteration. These two ingredients naturally lead to the notions introduced below.

The first ingredient is to shift the perspective of propagating $\rho$ to propagating its representation over test functionals. 
\begin{defi}[Test functionals]\label{def:test_functions}
Let $k\ge0$. Denote by $\mathscr G^k$ the collection of functionals $G$ whose lifts $\tG$ are $k$ times Fr\'echet differentiable on $\bH$, and for which there exist constants $c^i \ge 0$ for all $i = 0 , \ldots, k$ such that
\begin{equation}
    |d^i_C\tG|_\mathrm{op} \leq c^i.
\end{equation}
Here $d_X^0\tG:=\tG(X)$ and $|d_X^0\tG|_{\mathrm{op}}:= |\tG(X)|$ while for $i\ge1$,
\[
|d_X^i\tG|_{\mathrm{op}}
:=
\sup_{H_1,\ldots,H_i\in\bH\setminus{0}}
\frac{
|d_X^i\tG[H_1,\ldots,H_i]|
}{
|H_1|_{\bH}\cdots|H_i|_{\bH}
}.
\]
In other words, $\mathscr G^k$ is the set of all functionals whose lifts are $k$ times Fr\'echet differentiable, with bounded derivatives in operator norm. 
We further define the compact-state seminorm
\begin{equation}\label{eq:G3_compact_norm}
|G|_{\mathscr G^k,\calX}
:=
\sup_{\operatorname{supp}(X)\subset\calX}
\max_{0\le i\le k}
|d_X^i\tG|_{\mathrm{op}}.
\end{equation}
Since $\calX$ is compact and $|X|_{\bH}\le R_\calX$ whenever $\operatorname{supp}(X)\subset\calX$, the seminorm above is finite.
\end{defi}

For $k = 0$ we denote by $\mathscr G:= \mathscr G^0 $ the space of all continuous functionals on $\calP(\calX)$. 
To track the evolution of observables, we introduce the one-step transition operators:

\begin{equation}\label{eq:P_h_gamma}
P_h^\gamma:\mathscr G\to\mathscr{G}\,,\quad(P_h^\gamma G)(\rho)
:=
\mathbb E_\gamma
\Big[
G\big(
(\mathrm{Id}-hu_\gamma(\cdot,\rho))_\#\rho
\big)
\Big],
\end{equation}
and
\begin{equation}\label{eq:P_h_omega}
P_h^\omega:\mathscr G\to\mathscr{G}\,,\quad (P_h^\omega G)(\rho)
:=
\mathbb E_\omega
\Big[
G\big(
(\mathrm{Id}-hu_\omega(\cdot,\rho))_\#\rho
\big)
\Big].
\end{equation}
These operators propagate observables backward through one step of the stochastic and Gaussian dynamics, respectively.

The operators in~\eqref{eq:P_h_gamma} and~\eqref{eq:P_h_omega} are stochastic Koopman operators~\cite{wanner2022robust} acting on observables over $\calP_2(\calX)$. They are associated with two different discrete-time Markov dynamics on $\calP_2(\calX)$: $P_h^\gamma$ corresponds to the original stochastic update, whereas $P_h^\omega$ corresponds to the moment-matched Gaussian update. Thus, each operator propagates an observable backward through one step of its respective dynamics.


The second ingredient is to formalize the restarting procedure. Let $\rho\in\calP_2(\calX)$ and $0\le k\le n$. We define the restarted trajectories $\rho_j^Z(\rho,k)$ and $\rho_j^X(\rho,k)$, for $j=k,\ldots,n$, as the solutions of \eqref{eq:discr} and \eqref{eq:gauss} but initialized at
\[
\rho_k^Z(\rho,k)=\rho_k^X(\rho,k)=\rho\,.
\]
In this context, the notation $\rho_j^Z(\rho,k)$ or $\rho_j^X(\rho,k)$ represents the $j$-th iteration of the corresponding dynamics when the $k$-th iterate is denoted as $\rho$, with $j \geq k$. With this notation, the original SGD iterates and their Gaussian approximation, which are solutions to~\eqref{eq:discr} and~\eqref{eq:gauss}, are simply
\[
\rho_j^Z(\rho_0,0)
\qquad\text{and}\qquad
\rho_j^X(\rho_0,0),\qquad j = 1,\ldots,n.
\]
respectively. With these notations, define
\begin{equation}\label{eq:observabla_propagation}
\calV_n:=G, \qquad \calV_k:=P_h^\omega\calV_{k+1},
\qquad \text{for}\quad k=n-1,\ldots,0,
\end{equation}
we have
\[
\calV_{n-1}(\rho) = P_h^\omega\calV_{n}(\rho)=\mathbb{E}_{\omega_n}[\calV_n(\rho_n^X(\rho,n-1))]=\mathbb E_{\omega_{n}} \left[ G\bigl(\rho_n^X(\rho,n-1)\bigr) \right]\,,
\]
and by induction,
\begin{equation}\label{eq:propagation_observ_exp}
\calV_k(\rho)= \mathbb E_{\omega_{k+1:n}}
\left[
G\bigl(\rho_n^X(\rho,k)\bigr)
\right].
\end{equation}
This quantity characterizes, if at $k$-th iteration the state is prepared at $\rho$, the average reading of $\rho_n^X$ at $n$-th iteration by $G$. Here we use abbreviated notation $\mathbb E_{\gamma_{a:b}}$ and $\mathbb E_{\omega_{a:b}}$ to denote the expectation over $\gamma_a,\ldots,\gamma_b$ and over $\omega_a,\ldots,\omega_b$, respectively. When $a>b$, the corresponding expectation is omitted. Later we will also use
\[
\mathbb E_{\gamma_{<k}}:=\mathbb E_{\gamma_{1:k-1}},
\qquad
\mathbb E_{\omega_{<k}}:=\mathbb E_{\omega_{1:k-1}}.
\]
When no range is specified, $\mathbb E_\gamma$ and $\mathbb E_\omega$ denote expectation over all random variables relevant to the expression.

\subsection{Main results}
We now present our main result. A list of assumptions are needed.
\begin{assumption}\label{ass:main}
There exist constants
\[
h_0>0,\qquad M_\gamma<\infty,\qquad M_\omega<\infty,
\qquad \Lambda<\infty
\]
such that the following conditions hold for all $0<h\le h_0$.

\begin{enumerate}
\item[(A1)] For every $\rho\in\calP_2(\calX)$, the maps
\[
x\mapsto x-hu_\gamma(x,\rho),
\qquad
x\mapsto x-hu_\omega(x,\rho)=x-hU(x,\rho;\xi)
\]
send $\calX$ into $\calX$, for $\mathbb P$-a.e. $\gamma$ and $\Pi_\rho$-a.e. $\xi$, respectively.
\item[(A2)] For every $\rho\in\calP_2(\calX)$ and every lift $X\in\bH$ so that $X_\#\bP_\theta=\rho$,
\[
\mathbb E_\gamma
\big[
\|u_\gamma(X,\rho)\|_{\bH}^3
\big]
\le M_\gamma,
\quad \text{and}\quad
\mathbb E_\xi
\big[
\|U(X,\rho;\xi)\|_{\bH}^3
\big]
\le M_\omega .
\]


\item[(A3)] For every $G\in\mathscr G^3$,we have that $P_h^\omega G\in\mathscr G^3$ (defined in~\eqref{eq:P_h_omega}), and that
\[
\|P_h^\omega G\|_{\mathscr G^3,\calX}\le (1+\Lambda h)\|G\|_{\mathscr G^3,\calX}\,.
\]
\end{enumerate}
\end{assumption}

\begin{rmk}
A sufficient condition for Assumption~\ref{ass:main}(A2) regarding $M_\omega$ is
\[
\sup_{\rho\in\calP_2(\calX)}
\int_{\calX}
|\bar u_\rho(x)|^3\,d\rho(x)
<\infty,\quad \text{and}\quad 
\sup_{\rho\in\calP_2(\calX)}
\int_{\calX}
\big(\operatorname{tr}K_\rho(x,x)\big)^{3/2}\,d\rho(x)
<\infty .
\]
\end{rmk}

\begin{rmk}
    The following additional conditions on the realizations of the Gaussian field $u_\omega(\cdot, \rho) = U(\cdot, \rho;\xi)$ are sufficient to guarantee that Assumption \ref{ass:main}(A3) is satisfied. 
    \begin{enumerate}
        \item (Differentiablility in the first argument) For almost all $\omega$, and for all fixed $\rho \in \mathcal{P}(\calX)$, the function $x \to u_\omega(x, \rho):\calX \to \calX$ is three times continuously differentiable, with derivatives denoted by $\partial_x^iu_\omega(x, \rho)$.
        \item (Lions differentiability in the second argument) For almost all $\omega$, and for all fixed $x \in \calX$, the function $\rho \to u_\omega(x, \rho):\mathcal{P}(\calX) \to \calX$ is three times differentiable in the Lions sense (see Definition~\ref{def:lions}) with continuous derivatives (in the weak$^*$ topology) denoted by $\partial_\rho u_\omega(x, \rho)$.
        \item (Continuity and boundedness) All mixed derivatives $\partial^i_x\partial^j_\rho u_\omega(x, \rho): \bH^j \times \calX^i \to \calX$, $i + j \le 3$ are jointly continuous in $x$ and $\rho$, and are bounded in operator norm. Denote the upper bounds by:
        $$\begin{cases}
           \sup_{\rho \in \mathcal{P}(\calX), x \in \calX} \Big(| \partial_x u_\omega|_\mathrm{op}  + |\partial_\rho u_\omega|_\mathrm{op} \Big)\le C_1\\
            \sup_{\rho \in \mathcal{P}(\calX), x \in \calX}\Big(| \partial^2_x u_\omega|_\mathrm{op}  +  2|\partial_x\partial_\rho u_\omega |_\mathrm{op}  +  | \partial^2_\rho u_\omega|_\mathrm{op} \Big) \le C_2 \\
             \sup_{\rho \in \mathcal{P}(\calX), x \in \calX}\Big(|  \partial^3_x u_\omega|_\mathrm{op} + 3| \partial^2_x\partial_\rho u_\omega|_\mathrm{op}  + 3 | \partial_x\partial^2_\rho u_\omega|_\mathrm{op}  +  |\partial^3_\rho u_\omega |_\mathrm{op} \Big)\le C_3
        \end{cases}$$
    \end{enumerate}
    Under these additional assumptions, the growth constant $\Lambda $ is upper bounded by:
    \[
        \Lambda \leq 2C_1 + 3C_1C_2 + 3C_2 + C_1^2 + C_3 , \quad \text{ for } h \le 1.
    \]
    
\end{rmk}
\begin{thm}[Main result]\label{thm:main}
Let Assumption~\ref{ass:main} hold. Fix a terminal time $T>0$, let $0<h$ be the step-size, and set $n=\lfloor T/h\rfloor$ to be the iteration number. Then, for every terminal observable $G\in\mathscr G^3$ and every initial measure $\rho_0\in\calP_2(\calX)$,
\[
\left|
\mathbb E_\omega
\big[
G(\rho_n^X(\rho_0,0))
\big]
-
\mathbb E_\gamma
\big[
G(\rho_n^Z(\rho_0,0))
\big]
\right|
\le
C_T h^2\,,
\]
where $\rho_n^X$ and $\rho_n^Z$ are generated by following~\eqref{eq:discr} and~\eqref{eq:gauss} starting from the common initial $\rho_0$. The constant is independent of $h$ and $n$ and takes on a specific form:
\[
C_T =\frac{T}{6}
(M_\omega+M_\gamma)
e^{\Lambda T}
\|G\|_{\mathscr G^3,\calX}\,.
\]
\end{thm}
Despite the new geometry and the highly technical use of Lions lifting and differentiability, the result is very intuitive. The matching of the two velocity fields is secured at the level of the mean and the covariance, suggesting the leading-order error should be of order $h^2$. Furthermore, the procedure does not distinguish dissipative or contracting dynamics, so exponential growth in time is also expected. However, gradient flow naturally descends along gradients, so the system is expected to contract, and one may potentially turn the exponential growth into decay. This has not been pursued in this paper.

\begin{figure}
    \centering

\begin{tikzpicture}[x=1cm,y=1cm,>=stealth]

\definecolor{pastblue}{RGB}{74,144,226}
\definecolor{sgdgreen}{RGB}{65,117,5}
\definecolor{gaussred}{RGB}{208,2,27}

\coordinate (A)  at (0,0);
\coordinate (Gk) at (2.8,1.05);
\coordinate (Wk) at (2.8,-1.05);
\coordinate (Gnhi) at (7.4,1.30);
\coordinate (Gn)   at (7.4,0.98);
\coordinate (Gnlo) at (7.4,0.64);
\coordinate (Wnhi) at (7.4,-0.64);
\coordinate (Wn)   at (7.4,-0.98);
\coordinate (Wnlo) at (7.4,-1.30);

\filldraw[fill=sgdgreen!10,draw=sgdgreen!45]
    (Gn) ellipse (0.14 and 0.53);
\filldraw[fill=gaussred!10,draw=gaussred!45]
    (Wn) ellipse (0.14 and 0.53);

\draw[densely dotted] (0,-1.85) -- (0,1.85);
\draw[densely dotted] (2.8,-1.85) -- (2.8,1.85);
\draw[densely dotted] (7.4,-1.85) -- (7.4,1.85);

\draw[-{Latex[length=2mm, width=2mm]},sgdgreen]
    (A) -- (Gk)
    node[midway,above left] {\small SGD step};

\draw[-{Latex[length=2mm, width=2mm]},gaussred]
    (A) -- (Wk)
    node[midway,below left] {\small Gaussian step};

\draw[-{Latex[length=2mm, width=2mm]},sgdgreen,dashed]
    (Gk) .. controls (4.0,1.65) and (5.8,1.15) .. (Gn);

\draw[-{Latex[length=2mm, width=2mm]},gaussred,dashed]
    (Wk) .. controls (4.0,-1.65) and (5.8,-1.15) .. (Wn);

\draw[sgdgreen!45,dashed]
    (Gk) .. controls (4.0,0.55) and (5.9,0.50) .. (Gnlo);
\draw[sgdgreen!45,dashed]
    (Gk) .. controls (4.2,1.35) and (5.6,1.48) .. (Gnhi);

\draw[gaussred!45,dashed]
    (Wk) .. controls (4.0,-0.55) and (5.9,-0.50) .. (Wnhi);
\draw[gaussred!45,dashed]
    (Wk) .. controls (4.2,-1.35) and (5.6,-1.48) .. (Wnlo);

\fill[pastblue] (A) circle (2.4pt);
\fill[sgdgreen] (Gk) circle (2.4pt);
\fill[gaussred] (Wk) circle (2.4pt);

\fill[sgdgreen!30] (7.4,0.50) circle (1.0pt);
\fill[sgdgreen!50] (7.4,0.64) circle (1.3pt);
\fill[sgdgreen!75] (7.4,0.80) circle (1.7pt);
\fill[sgdgreen]    (Gn)       circle (2.4pt);
\fill[sgdgreen!75] (7.4,1.15) circle (1.7pt);
\fill[sgdgreen!50] (7.4,1.30) circle (1.3pt);
\fill[sgdgreen!30] (7.4,1.46) circle (1.0pt);

\fill[gaussred!30] (7.4,-0.50) circle (1.0pt);
\fill[gaussred!50] (7.4,-0.64) circle (1.3pt);
\fill[gaussred!75] (7.4,-0.80) circle (1.7pt);
\fill[gaussred]    (Wn)        circle (2.4pt);
\fill[gaussred!75] (7.4,-1.15) circle (1.7pt);
\fill[gaussred!50] (7.4,-1.30) circle (1.3pt);
\fill[gaussred!30] (7.4,-1.46) circle (1.0pt);

\node[below] at (0,-1.85) {$k-1$};
\node[below] at (2.8,-1.85) {$k$};
\node[below] at (7.4,-1.85) {$n$};

\node[left,pastblue] at (A) {$\rho_{k-1}^Z$};

\node[above,sgdgreen] at (Gk) {$\rho_k^\gamma$};
\node[below,gaussred] at (Wk) {$\rho_k^\omega$};

\node[right,sgdgreen] at (7.62,0.98)
    {\small $\operatorname{Law}_\omega[\rho_n^X\mid\rho_k^\gamma]$};

\node[right,gaussred] at (7.62,-0.98)
    {\small $\operatorname{Law}_\omega[\rho_n^X\mid\rho_k^\omega]$};

\node[sgdgreen,align=center] at (5.1,1.75)
    {\small representative future\\[-0.2em]\small Gaussian sample paths};

\node[gaussred,align=center] at (5.1,-1.75)
    {\small representative future\\[-0.2em]\small Gaussian sample paths};

\end{tikzpicture}
\caption{
Schematic representation of one summand in the telescoping identity \eqref{eq:last}. Conditioned on the past $\gamma_{<k}$, both branches start from the same measure $\rho_{k-1}^Z$. The red branch performs one Gaussian update, $\rho_k^\omega=\bigl(\mathrm{Id}-hU(\cdot,\rho_{k-1}^Z;\xi_k)\bigr)_\#\rho_{k-1}^Z$, whereas the green branch performs one SGD update, $\rho_k^\gamma = \bigl(\mathrm{Id}-hu_{\gamma_k}(\cdot,\rho_{k-1}^Z)\bigr)_\#\rho_{k-1}^Z$. From time $k$ to time $n$, both branches are evaluated by the same observable $\calV_k$, equivalently by averaging $G$ over all future Gaussian sample paths. The dashed curves show representative paths, and the endpoint clouds in the $n$-column depict their conditional terminal laws.
}
\label{fig:summands}
\end{figure}
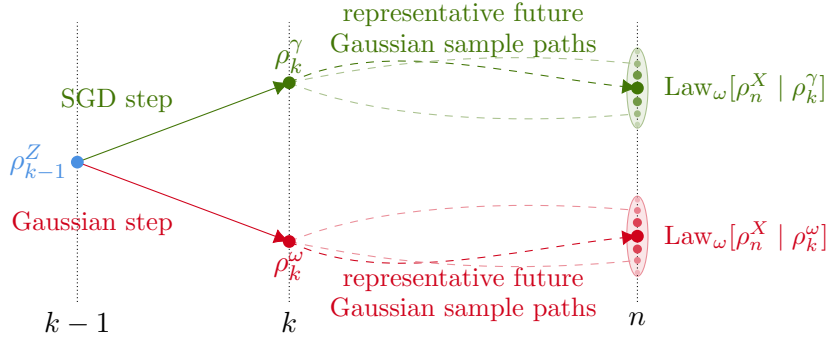

\subsection{Proof of the main result}





\begin{proof}[Proof of Theorem~\ref{thm:main}]

By Assumption~\ref{ass:main}(A3),
\[
\|\calV_k\|_{\mathscr G^3,\calX}
\le
(1+\Lambda h)^{n-k}\|G\|_{\mathscr G^3,\calX}
\le
e^{\Lambda T}\|G\|_{\mathscr G^3,\calX},
\qquad 0\le k\le n.
\]
Set $G_T:=e^{\Lambda T}\|G\|_{\mathscr G^3,\calX}$. Thus, $\|\calV_k\|_{\mathscr G^3,\calX}\le G_T$, for $0\le k\le n$.

Note that, according to~\eqref{eq:propagation_observ_exp}, $\mathbb E_\omega \bigl[ G(\rho_n^X(\rho_0,0)) \bigr]= \calV_0(\rho_0)$, so by telescoping:
\begin{equation}\label{eq:last}
\begin{aligned}
&\mathbb E_\omega\bigl[ G(\rho_n^X(\rho_0,0)) \bigr] - \mathbb E_\gamma \bigl[G(\rho_n^Z(\rho_0,0))\bigr]\\
&\qquad= \calV_0(\rho_0) - \mathbb E_\gamma \bigl[\calV_n(\rho_n^Z) \bigr]\\
&\qquad=\sum_{k=1}^n\mathbb E_{\gamma_{<k}} \left[\calV_{k-1}(\rho_{k-1}^Z) -\mathbb E_{\gamma_k}\bigl[\calV_k(\rho_k^Z)\,\big|\,\rho_{k-1}^Z\bigr]\right]\\
&\qquad=\sum_{k=1}^n
\mathbb E_{\gamma_{<k}}
\left[
(P_h^\omega\calV_k)(\rho_{k-1}^Z)
-
(P_h^\gamma\calV_k)(\rho_{k-1}^Z)
\right],
\end{aligned}
\end{equation}
where we used the definition of $\calV_k$ to replace the first term and $P_h^\gamma$ in~\eqref{eq:P_h_gamma} to replace the second term. 
Figure~\ref{fig:summands} illustrates the two one-step updates appearing in one summand of \eqref{eq:last}. The future dashed paths should be read as representative Gaussian sample paths inside the expectation defining $\calV_k$.

Equation~\eqref{eq:P_h_gamma} nicely translates the total trajectory disparity into many one-step error estimates. We now need to show every single term in the summation in the last row of~\eqref{eq:last} is small. Fix $k\in\{1,\ldots,n\}$ and fix $\rho\in\calP_2(\calX)$. Let $X\in\bH$ be a lift of $\rho$. Let $\widetilde{\calV}_k$ denote a law-invariant lift of $\calV_k$. Then:
\[
(P_h^\omega\calV_k)(\rho)
=
\mathbb E_\omega
\left[
\widetilde{\calV}_k(X-hu_\omega(X,\rho))
\right]\,,\quad (P_h^\gamma\calV_k)(\rho)
=
\mathbb E_\gamma
\left[
\widetilde{\calV}_k(X-hu_\gamma(X,\rho))
\right]\,.
\]
Using the Taylor expansion derived in the Lions sense from Theorem~\ref{thm:Taylor},
\[
\widetilde{\calV}_k(X-hH)
=
\widetilde{\calV}_k(X)
-
h\,d_X\widetilde{\calV}_k[H]
+
\frac{h^2}{2}\,d_X^2\widetilde{\calV}_k[H,H]
+
R_k(X,H),
\]
where
\[
R_k(X,H)
=
\int_0^1
\frac{(1-\tau)^2}{2}
d^3_{X-\tau hH}\widetilde{\calV}_k[-hH,-hH,-hH]\,d\tau .
\]
Therefore
\[
(P_h^\omega\calV_k)(\rho)
-
(P_h^\gamma\calV_k)(\rho)
=
\mathrm{I}_k(\rho)+\mathrm{II}_k(\rho)+\mathrm{III}_k(\rho),
\]
where the three terms are:
\begin{itemize}
    \item Leading order disparity:
    \[
\mathrm{I}_k(\rho)
=
-h\,d_X\widetilde{\calV}_k
\left[
\mathbb E_\omega u_\omega(X,\rho)-\mathbb E_\gamma u_\gamma(X,\rho)
\right]=0\,,
\]
according to Assumption~\ref{ass:main}(A3), $\mathbb E_\omega H_\omega = \mathbb E_\gamma H_\gamma = \bar u_\rho\circ X$.
\item Second order disparity:
\[
\mathrm{II}_k(\rho)= \frac{h^2}{2}\left(\mathbb E_\omega \left[ d_X^2\widetilde{\calV}_k[u_\omega,u_\omega]\right] -\mathbb E_\gamma \left[ d_X^2\widetilde{\calV}_k[u_\gamma,u_\gamma]\right]
\right)=0\,,
\]
and this comes from Assumption~\ref{ass:main}(A3) and the application of Proposition~\ref{prop:bilinear_forms_and_covariance} on bilinear operator.
\item Third order disparity is:
\[
\mathrm{III}_k(\rho)
=
\mathbb E_\omega R_k(X,u_\omega)
-
\mathbb E_\gamma R_k(X,u_\gamma).
\]
By Assumption~\ref{ass:main}(A1), the points $X-hH_\omega$ and $X-hH_\gamma$ belong to $\bH$ almost surely. Since $\calX$ is convex, the whole segments $X-\tau hu_\omega$ and $X-\tau hu_\gamma$ are in $\bH$ for all $0\le\tau\le1$. Hence
\[
\max\left\{\|d^3_{X-\tau hu_\omega}\widetilde{\calV}_k\|_{\mathrm{op}}\,,\|d^3_{X-\tau hu_\gamma}\widetilde{\calV}_k\|_{\mathrm{op}}\right\}
\le \|\calV_k\|_{\mathscr G^3,\calX}\le G_T,
\]
and similarly for $H_\gamma$.
\end{itemize}
Putting these all together, and use Assumption~\ref{ass:main}(A2), we obtain, for every $\rho\in\calP_2(\calX)$,
\begin{equation}\label{eq:local_weak_error}
\left|
(P_h^\omega\calV_k)(\rho)
-
(P_h^\gamma\calV_k)(\rho)
\right|
\le
\frac{h^3}{6}
G_T
(M_\omega+M_\gamma).
\end{equation}

Combining \eqref{eq:last} and \eqref{eq:local_weak_error}, we get
\[
\begin{aligned}
&
\left|
\mathbb E_\omega
\bigl[
G(\rho_n^X(\rho_0,0))
\bigr]
-
\mathbb E_\gamma
\bigl[
G(\rho_n^Z(\rho_0,0))
\bigr]
\right|
\\
&\qquad
\le
\sum_{k=1}^n
\mathbb E_{\gamma_{<k}}
\left[
\left|
(P_h^\omega\calV_k)(\rho_{k-1}^Z)
-
(P_h^\gamma\calV_k)(\rho_{k-1}^Z)
\right|
\right]
\\
&\qquad
\le
n\,\frac{h^3}{6}G_T(M_\omega+M_\gamma).
\end{aligned}
\]
Since $nh\le T$, this gives
\[
\left|
\mathbb E_\omega
\bigl[
G(\rho_n^X(\rho_0,0))
\bigr]
-
\mathbb E_\gamma
\bigl[
G(\rho_n^Z(\rho_0,0))
\bigr]
\right|
\le
\frac{T}{6}
(M_\omega+M_\gamma)
e^{\Lambda T}
\|G\|_{\mathscr G^3,\calX}
h^2.
\]
This proves the theorem.
\end{proof}

\section{Numerical experiments}\label{sec:numerics}

We present two complementary experiments for the stochastic iteration \eqref{eq:discr} and its moment-matched Gaussian approximation \eqref{eq:gauss}. The first is an exactly solvable benchmark that isolates the second-order weak error. The second is a stochastic inverse problem in which a probability distribution is reconstructed from randomly sampled moment measurements; in this case the velocity and its covariance depend on the current measure.

For a terminal observable $\mathcal G$ and $n=T/h$, we report
\begin{equation}\label{eq:numerical_weak_error}
    e_h(\mathcal G,T)
    :=\left|
    \mathbb E_\gamma[\mathcal G(\rho_n^Z)]
    -\mathbb E_\omega[\mathcal G(\rho_n^X)]
    \right|.
\end{equation}
All computations use a Euclidean projection $\pi_{\mathcal X}$ after each explicit update as a compact-domain safeguard. Projection activations are recorded in every Monte Carlo experiment; none occurred in the results below.

\subsection{Solvable quadratic benchmark}\label{subsec:numerical_quadratic}

Let $\mathcal X=[-10,10]$ and consider
\begin{equation}\label{eq:numerical_energy}
    \mathcal E_\gamma(\rho)
    :=\frac12\int_{\mathcal X}(x-\gamma)^2\,d\rho(x),
    \qquad
    u_\gamma(x,\rho)=x-\gamma.
\end{equation}
We use the skew two-point distribution
\begin{equation}\label{eq:numerical_noise}
    \mathbb P(\gamma=2)=\frac13,
    \qquad
    \mathbb P(\gamma=-1)=\frac23.
\end{equation}
It satisfies $\mathbb E\left[\gamma\right]=0$, $\mathbb E\left[\gamma^2\right]=2$, and $\mathbb E\left[\gamma^3\right]=2$. The moment-matched Gaussian parameter is therefore $\omega\sim\mathcal N(0,2)$, with velocity $u_\omega(x,\rho)=x-\omega$.

Writing
\[
    \pi_{\mathcal X}(x):=\min\{10,\max\{-10,x\}\},
\]
the implemented recursions are
\begin{equation}\label{eq:numerical_updates}
\begin{aligned}
    \rho_{k+1}^Z
    &=\bigl(\pi_{\mathcal X}\circ((1-h)\operatorname{Id}+h\gamma_k)\bigr)_\#
      \rho_k^Z,\\
    \rho_{k+1}^X
    &=\bigl(\pi_{\mathcal X}\circ((1-h)\operatorname{Id}+h\omega_k)\bigr)_\#
      \rho_k^X.
\end{aligned}
\end{equation}

We initialize at $\rho_0=\delta_0$ and take
\begin{equation}\label{eq:numerical_observable}
    \mathcal G(\rho):=\int_{\mathcal X}\sin(x)\,d\rho(x).
\end{equation}
For $0<h\leq1$, the original process never reaches the boundary because each update is a convex combination of the current state and $\gamma_k\in\{-1,2\}$. Consequently,
\[
    x_n^Z=h\sum_{j=1}^n(1-h)^{n-j}\gamma_j,
    \qquad \rho_n^Z=\delta_{x_n^Z},
\]
and independence gives
\begin{equation}\label{eq:numerical_exact_error}
    \mathbb E_\gamma[\mathcal G(\rho_n^Z)]
    =\operatorname{Im}
    \prod_{r=0}^{n-1}
    \left(
        \frac13e^{2ih(1-h)^r}
        +\frac23e^{-ih(1-h)^r}
    \right).
\end{equation}
The projected Gaussian recursion is symmetric about zero, so $\mathbb E_\omega[\mathcal G(\rho_n^X)]=0$. A cumulant expansion of \eqref{eq:numerical_exact_error}, with $T=nh$ fixed, yields
\begin{equation}\label{eq:numerical_asymptotic_error}
    e_h(\mathcal G,T)
    =C_*h^2+\mathcal O(h^3),
    \qquad
    C_*:=\frac{1-e^{-3T}}9.
\end{equation}
The nonzero coefficient is generated by the third-moment mismatch between the skew two-point noise and its Gaussian surrogate.

Figure~\ref{fig:compact_weak_convergence} directly compares the exact weak error at $T=1$ with numerical estimates obtained from simulated trajectories for $h=2^{-\ell}$, $\ell=2,\ldots,5$. At each reported step size, we simulate $10^6$ trajectories of each dynamics, organized into $5\times10^5$ independent pairs. The error bars are $95\%$ confidence intervals for the absolute weak error, and the exact value lies inside the interval at every reported resolution. No projection $\pi_{\mathcal X}$ was activated. 

This benchmark demosntrates the result in Theorem~\ref{thm:main}. Matching the mean and covariance removes the first- and second-order contributions in the one-step Taylor expansion, while the third-moment mismatch produces a local error of order $h^3$; over $n=T/h$ steps this accumulates to the global order $h^2$ predicted by the theorem. Because $C_*\neq0$, the exact expansion \eqref{eq:numerical_asymptotic_error} also shows that the second-order bound is sharp in general. 

\begin{figure}[tbp]
    \centering
    \includegraphics[width=0.82\linewidth]
        {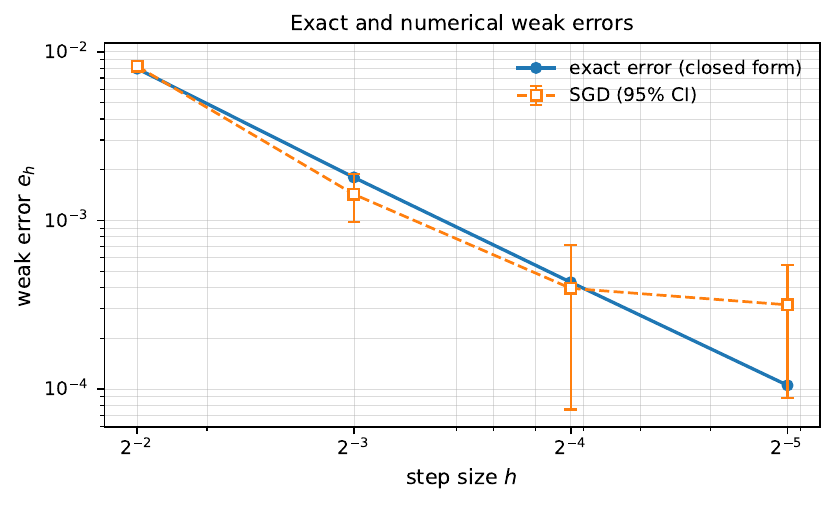}
    \caption{Solvable quadratic benchmark at $T=1$. The solid curve shows the exact weak error computed from the closed-form expectations, while the dashed squares show estimates obtained from simulated trajectories. Each numerical estimate uses $10^6$ trajectories for each dynamics, arranged into $5\times10^5$ independent pairs; the bars indicate $95\%$ confidence intervals.}
    \label{fig:compact_weak_convergence}
\end{figure}

\subsection{Random-moment distribution reconstruction}
\label{subsec:numerical_moments}

We next study an optimization problem that is to reconstruct a distribution that matches desired first two moments:
\begin{equation}\label{eq:moment_expected_objective}
    \calE(\rho)
    =\frac{p}{2}(m_\rho-y_1)^2
    +\frac{1-p}{2}(q_\rho-y_2)^2 = \E_{\gamma \sim \text{Bernouli}(p)} [\calE_{\gamma + 1}(\rho)],
\end{equation}
where $\gamma \in \{0, 1\}$ is a Bernouli random variable with parameter $p$, and 
\begin{equation}\label{eq:moment_sample_losses}
    \mathcal E_1(\rho):=\frac12(m_\rho-y_1)^2,
    \qquad
    \mathcal E_2(\rho):=\frac12(q_\rho-y_2)^2\,.
\end{equation}
Here $m_\rho$ and $q_\rho$ are the first two moments:
\[
    m_\rho:=\int_{\mathcal X}x\,d\rho(x),
    \qquad
    q_\rho:=\int_{\mathcal X}x^2\,d\rho(x)\,.
\]
Numerically we set the targets $y_1=0.75$ and $y_2=1$ and $p=1/3$. Conditioned on $\rho$, one can compute the mean function and covariance kernel: 
\begin{equation}\label{eq:moment_mean_covariance}
\begin{aligned}
    \bar u_\rho(x)
    &=pu_1(x,\rho)+(1-p)u_2(x,\rho),\\
    K_\rho(x,y)
    &=p(1-p)
      \bigl(u_1(x,\rho)-u_2(x,\rho)\bigr)
      \bigl(u_1(y,\rho)-u_2(y,\rho)\bigr),
\end{aligned}
\end{equation}
where the two velocities are
\begin{equation}\label{eq:moment_sample_velocities}
    u_1(x,\rho)=m_\rho-y_1,
    \qquad
    u_2(x,\rho)=2(q_\rho-y_2)x.
\end{equation}
Since the covariance matrix is rank one, the matching Gaussian field is represented by a single Gaussian variable:
\begin{equation}\label{eq:moment_gaussian_field}
    U(x,\rho;\xi)
    =\bar u_\rho(x)
    +\sqrt{p(1-p)}
    \bigl(u_1(x,\rho)-u_2(x,\rho)\bigr)\xi,
    \qquad \xi\sim\mathcal N(0,1)\,.
\end{equation}

We take $\mathcal X=[-20,20]$ and $\rho_0=\operatorname{Unif}[-1,1]$. The projection $\pi_{\mathcal X}$ is applied after each update. Both velocities in \eqref{eq:moment_sample_velocities} are affine in $x$; as long as the boundary is inactive, each realization is therefore an affine image of the initial uniform distribution. We propagate this affine representation and evaluate its moments analytically. 
The terminal observable is
\begin{equation}\label{eq:moment_terminal_observable}
    \mathcal G(\rho):=m_\rho^3.
\end{equation}
This cubic statistic is smooth and is sensitive to the first unmatched moment. For each step size $h$, we estimate the difference between two dynamics with $2\times10^6$ coupled paths. 

\begin{figure}[tbp]
    \centering
    \includegraphics[width=0.94\linewidth]
        {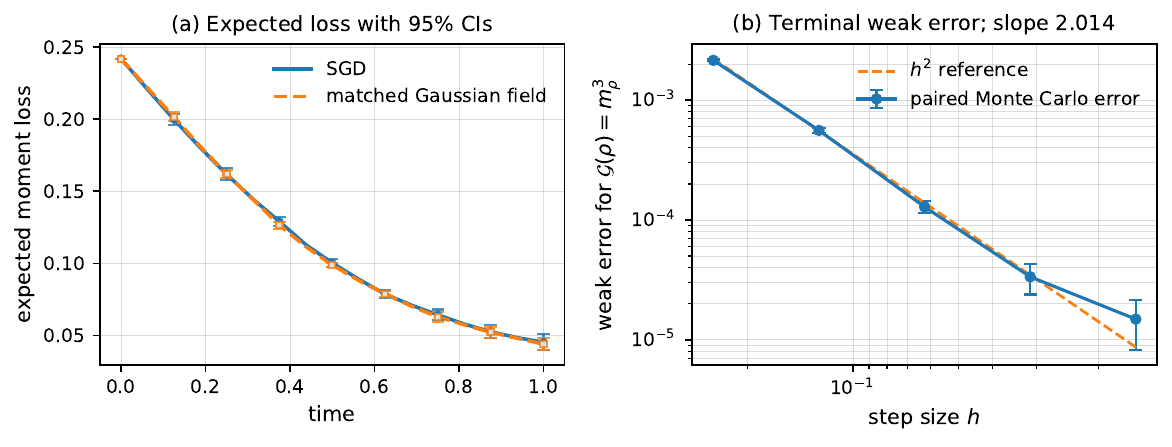}
    \caption{Random-moment distribution reconstruction.
    \textbf{(a)} Expected objective \eqref{eq:moment_expected_objective} for the SGD iteration and its matched Gaussian approximation at $h=2^{-4}$, estimated from $25$ independently sampled trajectories for each dynamics; markers and whiskers show pointwise $95\%$ confidence intervals at selected times. \textbf{(b)} Paired Monte Carlo estimate of the terminal weak error for $\mathcal G(\rho)=m_\rho^3$, with $95\%$ confidence intervals and a second-order reference line. The fitted order is $2.014$.}
\label{fig:moment_reconstruction}
\end{figure}

The results are presented in Figure~\ref{fig:moment_reconstruction}. The left panel of Figure~\ref{fig:moment_reconstruction} shows $\mathbb E[\mathcal E(\rho_t)]$ for $h=2^{-4}$, estimated from $25$ independently sampled trajectories for each dynamics. Pointwise $95\%$ Student-$t$ confidence intervals are shown at nine equally spaced times.  The right panel reports the terminal weak error at $T=1$ for $h=2^{-\ell}$, $\ell=2,\ldots,6$, using $2\times10^6$ paths per step size. A log-log fit over the four points with signal-to-noise ratio greater than five gives order $2.014$. The smallest-step result is shown with its $95\%$ confidence interval but is excluded from the fit because the Monte Carlo error dominates. No projection activation occurred in any of the reported trajectories.

This example tests the same conclusion in the state-dependent setting for which the analysis was developed. Equations~\eqref{eq:moment_mean_covariance} and~\eqref{eq:moment_gaussian_field} match the conditional mean and covariance of the velocity field at every current measure. The smooth cubic observable \eqref{eq:moment_terminal_observable} detects the remaining higher-moment discrepancy, and the fitted order $2.014$ is consistent with the $\mathcal O(h^2)$ estimate of Theorem~\ref{thm:main}. The overlapping objective curves further indicate that the Gaussian approximation dynamics reproduces the stochastic gradient dynamics.


\section{Conclusion}\label{sec:conclusion}

We have established a second-order weak approximation of stochastic Wasserstein gradient iterations by a discrete-time scheme driven by a state-dependent Gaussian velocity field. A Gaussian field that matches the conditional mean and covariance of the original stochastic velocity is capable of capturing the stochastic Wasserstein gradient dynamics with weak error of $h^2$ for sufficiently regular terminal observables, as discovered by Theorem~\ref{thm:main}.


The proof also identifies the mechanism behind this rate. A third-order Taylor expansion in the Lions sense reduces the one-step comparison to the first moments, covariance terms, and a cubic remainder. Moment matching makes the first- and second-order contributions agree, leaving a local discrepancy of order $h^3$. The telescoping weak-error representation and the stability of the Gaussian transition operators then accumulate these local errors over $T/h$ steps to produce the global second-order estimate. Thus, the Gaussian approximation retains precisely the fluctuation information needed at the order considered here.

The numerical experiments support and complement this analysis. In the solvable quadratic benchmark, the exact nonzero $h^2$ coefficient shows that the theoretical order is sharp in general. The random-moment reconstruction problem demonstrates the same behavior when both the velocity and its covariance depend on the evolving measure. In both cases, the compact-domain projection remained inactive, and the measured rates therefore reflect the interior moment-matching mechanism analyzed in the paper.

Several questions remain open. The present result relies on a compact-domain invariance condition, uniform third-moment bounds, and regularity propagation for the Gaussian transition operators. It would be useful to replace these assumptions with verifiable conditions for broader classes of energies, to treat unbounded domains and regimes in which a boundary mechanism is active, and to study the continuous-time stochastic Wasserstein dynamics suggested by the Gaussian approximation. Another direction is to investigate whether matching additional moments can yield higher-order weak schemes. These extensions would further clarify the role of Gaussian surrogates in stochastic optimization over probability measures. Lastly, we compare on the discrete-in-time level. Drawing inspiration from~\cite {li2019stochastic}, it is possible to pass the $h\to0$ limit to devise a modified stochastic differential PDE. However, this will necessarily trigger the study of the well-posedness of SPDEs over $\mathcal P_2(\mathcal X)$ when the velocity field is a Gaussian field that nonlinearly depends on the probability state. It is not yet evident.

\bibliographystyle{alpha}
\bibliography{sample}

\end{document}